%% file: sigkdd_outlier.tex
\documentclass[preprint,12pt]{elsarticle}

\usepackage{amsmath}
\usepackage{amsthm}
\usepackage{bm}
\usepackage{xcolor}
\usepackage{makecell}
\usepackage{tabularx}
\usepackage{float}
\usepackage{stfloats}
\usepackage{array}
\usepackage{threeparttable}
\usepackage{graphicx}

\input{Macros}
\usepackage[justification=centering]{caption}

\usepackage[ruled]{algorithm2e}
\usepackage{multirow}
\usepackage{amssymb}

\usepackage{subcaption,graphicx}

\makeatletter
\newcommand{\removelatexerror}{\let\@latex@error\@gobble}
\makeatother

\usepackage{url}

\begin{document}
%\setcopyright{acmcopyright}
%\setcopyright{acmlicensed}
%\setcopyright{rightsretained}
%\setcopyright{usgov}
%\setcopyright{usgovmixed}
%\setcopyright{cagov}
%\setcopyright{cagovmixed}

\begin{frontmatter}
\title{A Local Density-Based Approach for Local Outlier Detection}

\author{Bo Tang and Haibo He}

\address{Bo Tang and Haibo He are with the Department of Electrical, Computer, and Biomedical Engineering, University of Rhode Island, Kingston, RI, 02881 USA, Email: \textit{btang, he}@ele.uri.edu. \\
}

%\additionalauthors{Additional authors: John Smith (The Th{\o}rvald Group,
%email: {\texttt{jsmith@affiliation.org}}) and Julius P.~Kumquat
%(The Kumquat Consortium, email: {\texttt{jpkumquat@consortium.net}}).}
%\maketitle
\begin{abstract}
This paper presents a simple but effective density-based outlier detection approach with the local kernel density estimation (KDE). A \textit{Relative Density-based Outlier Score} (RDOS) is introduced to measure the local outlierness of objects, in which the density distribution at the location of an object is estimated with a local KDE method based on extended nearest neighbors of the object. Instead of using only $k$ nearest neighbors, we further consider reverse nearest neighbors and shared nearest neighbors of an object for density distribution estimation. Some theoretical properties of the proposed RDOS including its expected value and false alarm probability are derived. A comprehensive experimental study on both synthetic and real-life data sets demonstrates that our approach is more effective than state-of-the-art outlier detection methods.
\end{abstract}

\end{frontmatter}

%\begin{CCSXML}
%<ccs2012>
%  <concept_id>10010520.10010553.10010562</concept_id>
%  <concept_desc>Computer systems organization~Embedded systems</concept_desc>
%  <concept_significance>500</concept_significance>
%  <concept_id>10010520.10010575.10010755</concept_id>
%  <concept_desc>Computer systems organization~Redundancy</concept_desc>
%  <concept_significance>300</concept_significance>
%  <concept_id>10010520.10010553.10010554</concept_id>
%  <concept_desc>Computer systems organization~Robotics</concept_desc>
%  <concept_significance>100</concept_significance>
%  <concept_id>10003033.10003083.10003095</concept_id>
%  <concept_desc>Networks~Network reliability</concept_desc>
%  <concept_significance>100</concept_significance>
%</ccs2012>  
%\end{CCSXML}
%
%\ccsdesc[500]{Computer systems organization~Embedded systems}


%
%  Use this command to print the description
%
%\printccsdesc

% We no longer use \terms command

%\keywords{Local outlier detection; reverse nearest neighbors; shared nearest neighbors; local kernel density estimation}

\section{Introduction}
Advances in data acquisition have created massive collections of data, capturing valuable information to science, government, business, and society. However, despite of the availability of large amount of data, some events are rare or exceptional, which are usually called ``outliers" or ``anomalies". Compared with many other knowledge discovery problems, outlier detection is sometimes more valuable in many applications, such as network intrusion detection, fraudulent transactions, and medical diagnostics. For example, in network intrusion detection, the number of intrusions or attacks (``bad" connections) is much less than the ``good" and normal connections. Similarly, the abnormal behaviors are usually rare in many other cases. Although these outliers are only a small portion of the whole data set, it is much more costly to misunderstand them compared with other events. 

In recent decades, many outlier detection approaches have been proposed. Usually an outlier detection method can be categorized into the following four types of method \cite{jin2001mining}\cite{1334558}: distribution-based, distance-based, clustering-based, and density-based. In distribution-based methods, an object is considered as the outlier if it deviates from a standard distribution (e.g., normal, Poisson, etc.) too much \cite{barnett1994outliers}. The problem of the distribution-based method is that the underlying distribution is usually unknown and does not follow a standard distribution for many practical applications. 

The distance-based methods detect outliers by computing distances among all objects. An object is considered as the outlier when it has $d_0$ distance away from $p_0$ percentage of objects in the data set \cite{knox1998algorithms}. In \cite{aggarwal2001outlier}, the distance among objects is calculated in feature subspace through projections for high dimensional data sets. The problem of these methods is that the local outliers are usually misdetected for the data set with multiple components or clusters. To detect the local outliers, a top-$n$ $k$-th nearest neighbor distance is proposed in \cite{ramaswamy2000efficient}, in which the distance from an object to its $k$-th nearest neighbor indicates outlierness of the object. The cluster-based methods detect the outlier in the process of finding clusters. The object does not belong any cluster is considered as the outlier \cite{ester1996density}\cite{zhang1997birch}\cite{brito1997connectivity}. 

In density-based methods, an outlier is detected when its local density differs from its neighborhood. Different density estimation methods can be applied to measure the density. In Local Outlier Factor (LOF) \cite{breunig2000lof}, an outlierness score, indicating how an object differs from its locally reachable neighborhood, is measured. Previous studies \cite{tang2002enhancing}\cite{zhang2009new} have shown that it is more reliable to consider the objects with the highest LOF scores as outliers, instead of comparing the LOF score with a threshold. Several variations of the LOF are also proposed \cite{zhang2009new}\cite{jin2006ranking}. In \cite{zhang2009new}, a Local Distance-based Outlier Factor (LDOF) using the relative distance from an object to its neighbors is proposed for outlier detection in scattered datasets. In \cite{jin2006ranking}, a INFLuenced Outlierness (INFLO) score is measured by considering both neighbors and reverse neighbors of an object when estimating its relative density distribution \cite{jin2006ranking}. To address the issue that the LOF method and its variants do not consider the underlying pattern of data, Tang et. al. proposed a connectivity-based outlier factor (COF) scheme in \cite{tang2001robust}. While the LOF-based and COF-based outlier detection methods use the relative distance to estimate the density, several other density-based methods are proposed based on kernel density estimation \cite{latecki2007outlier}\cite{gao2011rkof}\cite{schubert2014generalized}. For
example, Local Density Factor (LDF) \cite{latecki2007outlier} extends the LOF by using kernel density estimation. In \cite{schubert2014generalized}, similar to the LOCI, a relative density score termed KDEOS is calculated using kernel density estimation and applies the $z$-score transformation for score normalization. 

In this paper, we propose an outlier detection method based on the local kernel density estimation for robust local outlier detection. Instead of using the whole data set, the density of an object is estimated with the objects in its neighborhood. Three kinds of neighbors: $k$ nearest neighbors, reverse nearest neighbors, and shared nearest neighbors, are considered in our local kernel density estimation. A simple but efficient relative density calculation, termed Relative Density-based Outlier Score (RDOS), is introduced to measure the outlierness. Theoretical properties of the RDOS, including the expected value and the false alarm probability are derived, which suggests parameter settings in practical applications. We further employ the top-$n$ scheme to rank the objects with their outlierness, i.e., the objects with the highest RDOS values are considered as the outliers. Simulation results on both synthetic data sets and real-life data sets illustrate superior performance of our proposed method. 

The paper is organized as follows: In Section 2, we introduce the definition of the RDOS and present the detailed descriptions of our proposed outlier detection approach. In Section 3, we derive theoretical properties of the RDOS and discuss the parameter settings. In Section 4, we present experimental results and analysis, which show superior performance compared with previous approaches. Finally, conclusions are given in Section 5.

%\section{Related Work}

\section{Proposed Outlier Detection}
\subsection{Local Kernel Density Estimation}
We use the KDE method to estimate the density at the location of an object based on the given data set. Given a set of objects $\mathcal{X} = \{X_1, X_2, \cdots, X_{m} \}$, where $X_i \in \mathbb{R}^{d}$ for $i=1,2,\cdots, m$, the KDE method estimates the distribution as follows:
\begin{align}
\label{KDE}
p(X) = \frac{1}{m} \sum_{i=1}^m \frac{1}{h^{d}} K\left(\frac{X - X_i}{h} \right) 
\end{align}
where $K\left(\frac{X - X_i}{h} \right)$ is the defined kernel function with the kernel width of $h$, which satisfies the following conditions:
\begin{align}
\int K(u) d u = 1, \int u K(u) d u = 0, \text{and} \int u^2 K(u) d u > 0
\end{align}

A commonly used multivariate Gaussian kernel function is given by
\begin{align}
K\left(\frac{X - X_i}{h} \right)_{\text{Gaussian}} = \frac{1}{ (2 \pi)^{d/2} } \exp \left( -\frac{\| X - X_i \|^2}{2h} \right)
\end{align} 
where $\| X - X_i \|$ denotes the Euclidean distance between $X$ and $X_i$. The distribution estimate in Eq. (\ref{KDE}) offers many nice properties, such as its non-parametric property, continuity and differentiability \cite{epanechnikov1969non}. Also it is an asymptotic unbiased estimator of the density. 

To estimate the density at the location of the object $X_p$, we only consider its neighbors of $X_p$ as kernels, instead of using all objects in the data set. The reason for this is twofold: firstly, many complex real-life data sets usually have multiple clusters or components, which are the intrinsic patterns of the data. The density estimation using the full data set may lose the local difference in density and fail to detect the local outliers; secondly, the outlier detection will calculate the score for each object, and using the full data set would lead to a high computational cost, which has the complexity of $O(N^2)$ where $N$ is the total number of objects in the data set. 

To better estimate the density distribution in the neighbourhood of an object, we propose to use $k$ nearest neighbors, reverse nearest neighbors and shared nearest neighbors as kernels in KDE. Let $NN_{r}(X_p)$ be the $r$-th nearest neighbors of the object $X_p$, we denote the set of $k$ nearest neighbors of $X_p$ as $\mathcal{S}_{KNN}(X_p)$:
\begin{align}
\mathcal{S}_{KNN}(X_p) = \{NN_{1}(X_p), NN_{2}(X_p), \cdots, NN_{k}(X_p)\}
\end{align}
The \textit{reverse nearest neighbors} of the object $X_p$ are those objects who consider $X_p$ as one of their $k$ nearest neighbors \cite{tangENN}, i.e., $X$ is one reverse nearest neighbor of $X_p$ if $NN_r(X) = X_p$ for any $r \leq k$. The \textit{shared nearest neighbors} of the object $X_p$ are those objects who share one or more nearest neighbors with $X_p$, in other words, $X$ is one shared nearest neighbor of $X_p$ if $NN_r(X) = NN_s(X_p)$ for any $r, s \leq k$. We show these three types of nearest neighbors in Fig. \ref{neighbors}.

\begin{figure}[!ht]
\centering
\includegraphics[width= 12 cm]{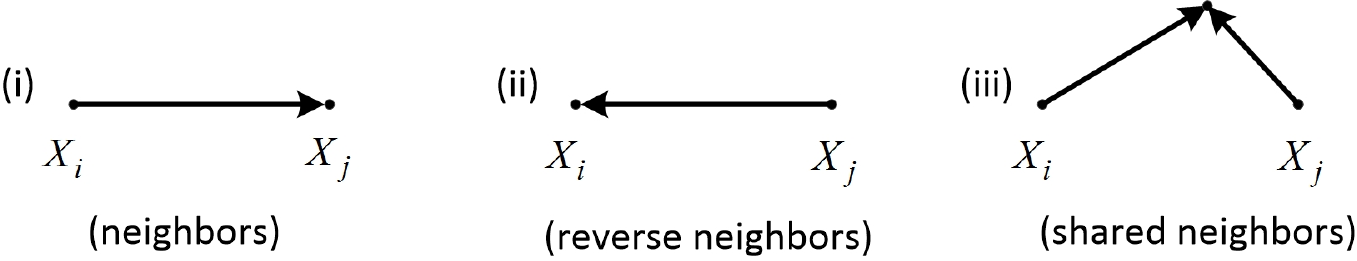}
\caption{Three types of nearest neighbors considered. Arrows from $X_i$ and $X_j$ to $NN_r(X_i)$ and $NN_s(X_j)$, respectively. }
\label{neighbors}
\end{figure}

We denote $\mathcal{S}_{RNN}(X_p)$ and $\mathcal{S}_{SNN}(X_p)$ by the sets of reverse nearest neighbors and shared nearest neighbors of $X_p$, respectively. For an object, there would be always $k$ nearest neighbors in $\mathcal{S}_{KNN}(X_p)$, while the sets of $\mathcal{RNN}(X_p)$ and $\mathcal{SNN}(X_p)$ can be empty or have one or more elements. Given the three data sets $\mathcal{S}_{KNN}(X_p)$, $\mathcal{S}_{RNN}(X_p)$ and $\mathcal{S}_{SNN}(X_p)$ for the object $X_p$, we form an extended local neighborhood by combining them together, denoted by $\mathcal{S}(X_p) = \mathcal{S}_{KNN}(X_p) \cup \mathcal{S}_{RNN}(X_p) \cup \mathcal{S}_{SNN}(X_p)$. Thus, the estimated density at the location of $X_p$ is written as
\begin{align}
p(X_p) = \frac{1}{| \mathcal{S}(X_p) | + 1} \sum_{X \in \mathcal{S}(X_p) \cup \{ X_p \}} \frac{1}{h^{d}} K\left(\frac{X - X_p}{h} \right)
\end{align}
where $| \mathcal{S}|$ denotes the number of elements in the set of $\mathcal{S}$. 

\subsection{Relative Density-based Outlier Factor}
After estimating the density at the locations of all objects, we propose a novel relative density-based outlier factor (RDOS) to measure the degree to which the density of the object $X_p$ deviates from its neighborhood, which is defined as follows:
\begin{align}
\label{RDOS}
RDOS_k(X_p) = \frac{\sum_{X_i \in \mathcal{S}(X_p)} p(X_i)}{|\mathcal{S}(X_p)| p(X_p)}
\end{align}
The RDOS is the ratio of the average neighborhood density to the density of interested object $X_p$. If $RDOS_k(X_p)$ is much larger than 1, then the object $X_p$ would be outside of a dense cluster, indicating that $X_p$ would be an outlier. If $RDOS_k(X_p)$ is equal or smaller than 1, then the object $X_p$ would be surrounded by the same dense neighbors or by a sparse cloud, indicating that $X_p$ would not be an outlier. In practice, we would like to rank the RDOS values and detect top-$n$ outliers. We summarize our algorithm in Algorithm \ref{algorithm_rdos}, which takes the KNN graph as input. The KNN graph is a directed graph in which each object is a vertex and is connected to its $k$ nearest neighbors with an outbound direction. In the KNN graph, an object will have $k$ outbound edges to the elements in $\mathcal{S}_{KNN}$, and have none, one or more inbound edges. The KNN graph construction using the brute-force method has the computational complexity of $O(N^2)$ for $N$ objects, and it can be reduced to $O(N \log N)$ using the $k-d$ trees \cite{bentley1975multidimensional}. Using the KNN graph KNN-G, it is easy to obtain the $k$ nearest neighbors $\mathcal{S}_{KNN}$, reverse nearest neighbors $\mathcal{S}_{RNN}$ and shared nearest neighbors $\mathcal{S}_{SNN}$ with an approximate computational cost of $O(N)$. For each object, we form a set of local nearest neighbors $\mathcal{S}$ with the combination of $\mathcal{S}_{KNN}$, $\mathcal{S}_{RNN}$ and $\mathcal{S}_{SNN}$, and calculate the density at the location of the object based on the set of $\mathcal{S}$. Then, we calculate the RDOS value of each object based on the densities of local neighbors in $\mathcal{S}$.  The top-$n$ outliers are obtained by sorting the RDOS values in a descending way. If one wants to determine whether an object $X_p$ is outlier, we can compare the value of $RDOS_k(X_p)$ with a threshold $\tau$, i.e., we determine an object is outlier if $RDOS_k(X_p)$ satisfies
\begin{align}
\label{outlier_threshold}
RDOS_k(X_p) > \tau
\end{align}
where the threshold $\tau$ is usually a constant value that is pre-determined by users.
		
\begin{figure}[!t]
 \removelatexerror
  \begin{algorithm}[H]
   \label{algorithm_rdos}
   \caption{RDOS for top-$n$ outlier detection based on the KNN graph}
   \textbf{INPUT: } $k$, $\mathcal{X}$, $d$, $h$, the KNN graph KNN-G.\\
   \textbf{OUTPUT: } Top-$n$ objects in $\mathcal{X}$.\\
	  
   \textbf{ALGORITHM:}\\  		
   \ForEach{object $X_p \in \mathcal{X}$}{
	\nl	 $\mathcal{S}_{KNN}(X_p) = getOutboundObjects(\text{KNN-G}, X_p)$: get $k$ nearest neighbors of $X_p$;\\
	\nl		$\mathcal{S}_{RNN}(X_p) = getInboundObjects(\text{KNN-G}, X_p)$: get reverse nearest neighbors of $X_p$;\\	
	\nl		$\mathcal{S}_{SNN}(X_p) = \emptyset$: initialize shared nearest neighbors of $X_p$;\\
	\nl		\ForEach{object $X \in \mathcal{S}_{KNN}(X_p)$}{
	\nl			$\mathcal{S}_{RNN}(X) = getInboundObjects(\text{KNN-G}, X)$;
	\nl			$\mathcal{S}_{SNN}(X_p) = \mathcal{S}_{SNN}(X_p) \cup \mathcal{S}_{RNN}(X)$: get objects who share $X$ as nearest neighbors with $X_p$;\\
		}
	\nl $\mathcal{S}(X_p) = \mathcal{S}_{KNN}(X_p) \cup \mathcal{S}_{RNN}(X_p) \cup \mathcal{S}_{SNN}(X_p)$;\\
	\nl $p(X_p) = getKernelDensity(\mathcal{S}(X_p), X_p, h)$: estimate the local kernel density at the location of $X_p$; 
   }
   
   \ForEach{object $X_p \in \mathcal{X}$}{
	\nl	Calculate $RDOS_k(X_p)$ for $X_p$ according to Eq. (\ref{RDOS}); 
   }
   
   \nl Sort RDOS in a descending way and output the top-$n$ objects.   
  \end{algorithm}
\end{figure}

\section{Theoretical Properties}
In this section, we analyze several nice properties of the proposed outlierness metric. In Theorem 1, we give the expected value of RDOS when the object and its neighbors are sampled from the same distribution, which indicates the lower bound of RDOS for outlier detection. 

\newtheorem{theorem}{Theorem}
\begin{theorem}
Let the object $X_p$ be sampled from a continuous density distribution. For $N\rightarrow \infty$, the RDOS equals 1 with probability 1, i.e., $RDOS_k(X_p) = 1$, when the kernel function $K$ is nonnegative and integrable.
\end{theorem}

\begin{proof}
For a fixed $k$, $N \rightarrow \infty$ indicates that the objects in $\mathcal{S}(X_p)$ locate in the local neighborhood of $X_p$ with the radius $r \rightarrow 0$.  Considering data sampled from a continuous density distribution $f(x)$, the expectation of the density estimation at $X_p$ exists and is consistent to the true one \cite{steckley2006estimating}:
\begin{align}
\mathbb{E}\left[ p(X_p) \right] = f(X_p) \int K(u) du = f(X_p)
\end{align}
and its asymptotic variance is given by \cite{steckley2006estimating}
\begin{align}
\mathbb{V}ar\left[ p(X_p) \right] = 0
\end{align}
Meanwhile, the average density at the neighborhood of $X_p$ with the radius of $r \rightarrow 0$ can be given by
\begin{align}
\mathbb{E}[\bar{p}(X_p)] = \mathbb{E}\left[ \frac{\sum_{X_i \in \mathcal{S}(X_p)} p(X_i)}{|\mathcal{S}(X_p)|} \right] = \mathbb{E}\left[ p(X_p) \right] = f(X_p)
\end{align}
Taking the ratio, we get 
\begin{align}
\mathbb{E}[\bar{p}(X_p)] / \mathbb{E}\left[ p(X_p) \right] = 1
\end{align}
\end{proof}

This theorem shows that when $RDOS_k(X_p) \approx 1$, we could say that the object $X_p$ is not an outlier. Since RDOS is always positive, when $ 0 < RDOS_k(X_p) < 1$, we could say the object $X_p$ can be ignored in outlier detection. Only these objects whose RDOS values are larger than 1 are possible to be outliers. 

Following the work in \cite{zhang2009new}, we next examine the upper-bound false detection probability to give a sense of threshold selection in practice. 

\begin{theorem}
Let $\mathcal{S}(X_p)$ be the set of local neighbors of $X_p$ in RDOS, which are assumed to be uniformly distributed in ball $B_r$ centered at $X_p$ with the radius of $r$. Using the Gaussian kernel, the probability of false detecting $X_p$ as an outlier is given by
\begin{align}
P[RDOS_k(X_p) > \gamma] \leq \exp \left( -\frac{2 (\gamma - 1)^2 (|\mathcal{S}| + 1)^2 (2\pi)^d h^{2d} }{|\mathcal{S}| (2|\mathcal{S}| + \gamma + 1)^2 V^2}\right)
\end{align}
where $h$ is the kernel width and $V$ is the volume of ball $B_r$.
\end{theorem}
\begin{proof}
For simplicity of notation, we use $\mathcal{S}$ for $\mathcal{S}(X_p)$ and consider $X_p = 0$. Then, the density estimation at $X_p$ given the local neighbors $X_1, X_2, \cdots, X_{|\mathcal{S}|}$ is written as
\begin{align}
p(X_p) = \frac{1}{|\mathcal{S}| + 1} \sum_{X_i \in \mathcal{S} \cup X_p} \frac{1}{ (2 \pi)^{d/2} h^{d}} \exp \left( -\frac{\| X_i \|^2}{2h} \right)
\end{align}

and the average density estimation in the neighborhood of $X_p$ is written as
\begin{align}
& \bar{p}(X_p) = \frac{1}{|\mathcal{S}|} \sum_{X_i \in \mathcal{S}} p(X_i) \nonumber \\
& = \frac{1}{|\mathcal{S}| (|\mathcal{S}| + 1) } \sum_{X_i \in \mathcal{S}} \sum_{X_j \in \mathcal{S} \cup X_p} \frac{1}{ (2 \pi)^{\frac{d}{2}} h^{d}} \exp \left( -\frac{\| X_i - X_j \|^2}{2h} \right)
\end{align}

For $X_i$, $i = 1, 2, \cdots, |\mathcal{S}|$, uniformly distributed in ball $B_r$, we can compute the expectation of both $p(X_p)$ and $\bar{p}(X_p)$ from Theorem 1, which is given by:
\begin{align}
\mathbb{E}[\bar{p}(X_p)] = \mathbb{E}[p(X_p)] = \frac{1}{V} = \frac{\pi^{n/2} r^{n}}{\Gamma(n/2 + 1)}
\end{align}
where $V$ is the volume of $n$-sphere $B_r$ and $n = d - 1$. The rest of proof follows the McDiarmid's Inequality which gives the upper bound of the probability that a function of i.i.d. variables $f(X_1, X_2, \cdots, X_{|\mathcal{S}|})$ deviates from its expectation. Let $f: \mathbb{R}^d \rightarrow \mathbb{R}$, $\forall i, \forall x_1, \cdots, x_{|\mathcal{S}|}, x_i^{'} \in \mathcal{S}$,
\begin{align}
|f(x_1, \cdots, x_i, \cdots, x_{|\mathcal{S}|}) - f(x_1, \cdots, x^{'}_i, \cdots, x_{|\mathcal{S}|}) | \leq c_i
\end{align}
Then, for all $\epsilon > 0$,
\begin{align}
\mathbb{P}[f - \mathbb{E}(f) \geq \epsilon ]  \leq \exp\left( \frac{-2\epsilon^2}{\sum_{i=1}^{|\mathcal{S}|} c^2_{i}}\right)
\end{align}
For $f_1 = p(X_p)$, we have
\begin{align}
& |f_1(x_1, \cdots, x_i, \cdots, x_{|\mathcal{S}|}) - f_1(x_1, \cdots, x_i^{'}, \cdots, x_{|\mathcal{S}|})| \nonumber \\
& = \frac{K(X_i/h) - K(X_i^{'}/h)}{ h^d (|\mathcal{S}| + 1)} \leq \frac{1 - \exp\left( -r^2 / 2h \right)}{(2\pi)^{d/2} h^d (|\mathcal{S}| + 1)} = c_{1}
\end{align}
For $f_2 = \bar{p}(X_p)$, we have
\begin{align}
& |f_2(x_1, \cdots, x_i, \cdots, x_{|\mathcal{S}|}) - f_2(x_1, \cdots, x_i^{'}, \cdots, x_{|\mathcal{S}|})| \nonumber \\
& = \frac{K(\frac{X_i}{h}) - K(\frac{X_i^{'}}{h}) + 2 \sum \limits_{j=1, j \neq i}^{|\mathcal{S}|} \left[ K(\frac{X_i - X_j}{h}) - K(\frac{X^{'}_i - X_j}{h}) \right] }{ h^d (|\mathcal{S}| + 1)} \nonumber \\
& \leq \frac{1 - \exp\left( -r^2 / 2h \right) + 2|\mathcal{S}| \left(1 - \exp\left( -2r^2 / h \right) \right) }{(2\pi)^{d/2} h^d (|\mathcal{S}| + 1)} = c_{2}
\end{align}
We define a new function $f = f_2 - \gamma f_1$, which is bounded by
\begin{align}
|f| \leq |f_2| + \gamma |f_1| \leq c_{2} + \gamma c_{1} \leq \frac{ 2|\mathcal{S}| + \gamma + 1 }{(2\pi)^{d/2} h^d (|\mathcal{S}| + 1)} = c
\end{align}
Then, the probability of false alarm is written as
\begin{align}
\mathbb{P}[RDOS_k(X_p) > \gamma ] & = \mathbb{P}[\bar{p}(X_p) - \gamma p(X_p) ] \nonumber \\
& = \mathbb{P}[f - \mathbb{E}(f) > t]
\end{align}
where $t = (\gamma - 1) / V$. From Theorem 1, we are only interested in the case of $RDOS_k(X_p) > 1$, i.e., $\gamma > 1$, and $t > 0$. Using the McDiarmid's Inequality, we have
\begin{align}
& \mathbb{P}[RDOS_k(X_p) > \gamma ] \leq \exp \left( -\frac{2 t^2}{\sum_{i=1}^{|\mathcal{S}|} c^2}\right) = \exp \left( -\frac{2 t^2}{|\mathcal{S}| c^2}\right) \nonumber \\
& \leq \exp \left( -\frac{2 (\gamma - 1)^2 (|\mathcal{S}| + 1)^2 (2\pi)^d h^{2d} }{|\mathcal{S}| (2|\mathcal{S}| + \gamma + 1)^2 V^2}\right)
\end{align}
\end{proof}

\section{Experimental Results and Analysis}
\subsection{Synthetic Data Sets}
We first test the proposed RDOS in two synthetic data sets for outlier detection. Our first synthetic data set includes two Gaussian clusters centered at $(0.5, 0.8)$ and $(2, 0.5)$, respectively, each of which has 100 data samples. There are three outliers around these two clusters, as indicated in Fig. \ref{example1}. To calculate the RDOS, we use $k = 21$ nearest neighbors and $h = 0.01$ in kernel functions. In Fig. \ref{example1_res}, we show the RDOS value of all data samples, where the color and the radius of circles denote the value of RDOS. It can be shown that the RDOS of these three outliers is significantly larger than that of non-outliers. 

\begin{figure}[!ht]
\centering
\includegraphics[width= 10 cm]{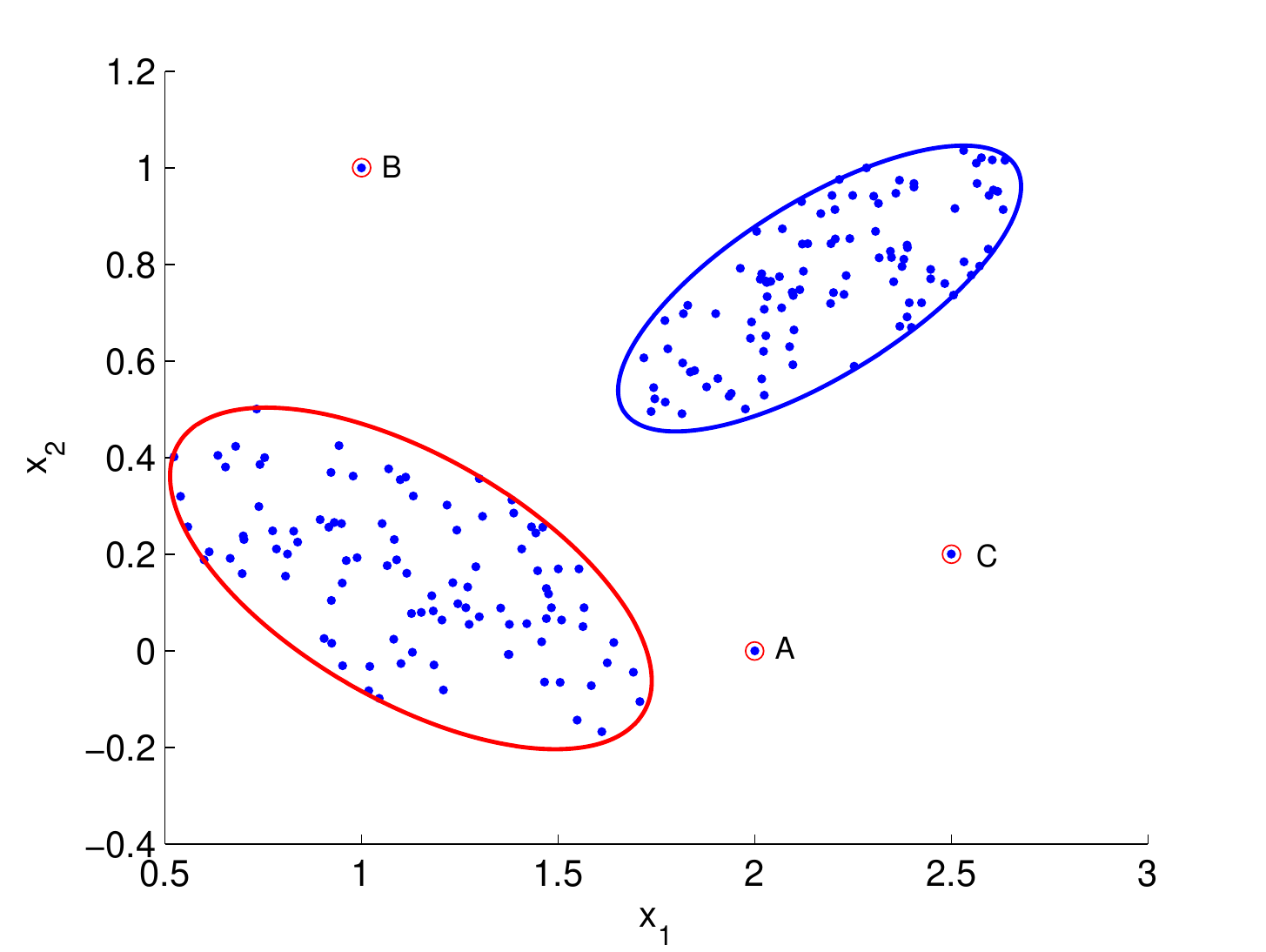}
\caption{The distribution of normal data and outliers, where the objects: $A$, $B$, and $C$ are outliers. }
\label{example1}
\end{figure}

\begin{figure}[!ht]
\centering
\includegraphics[width= 10 cm]{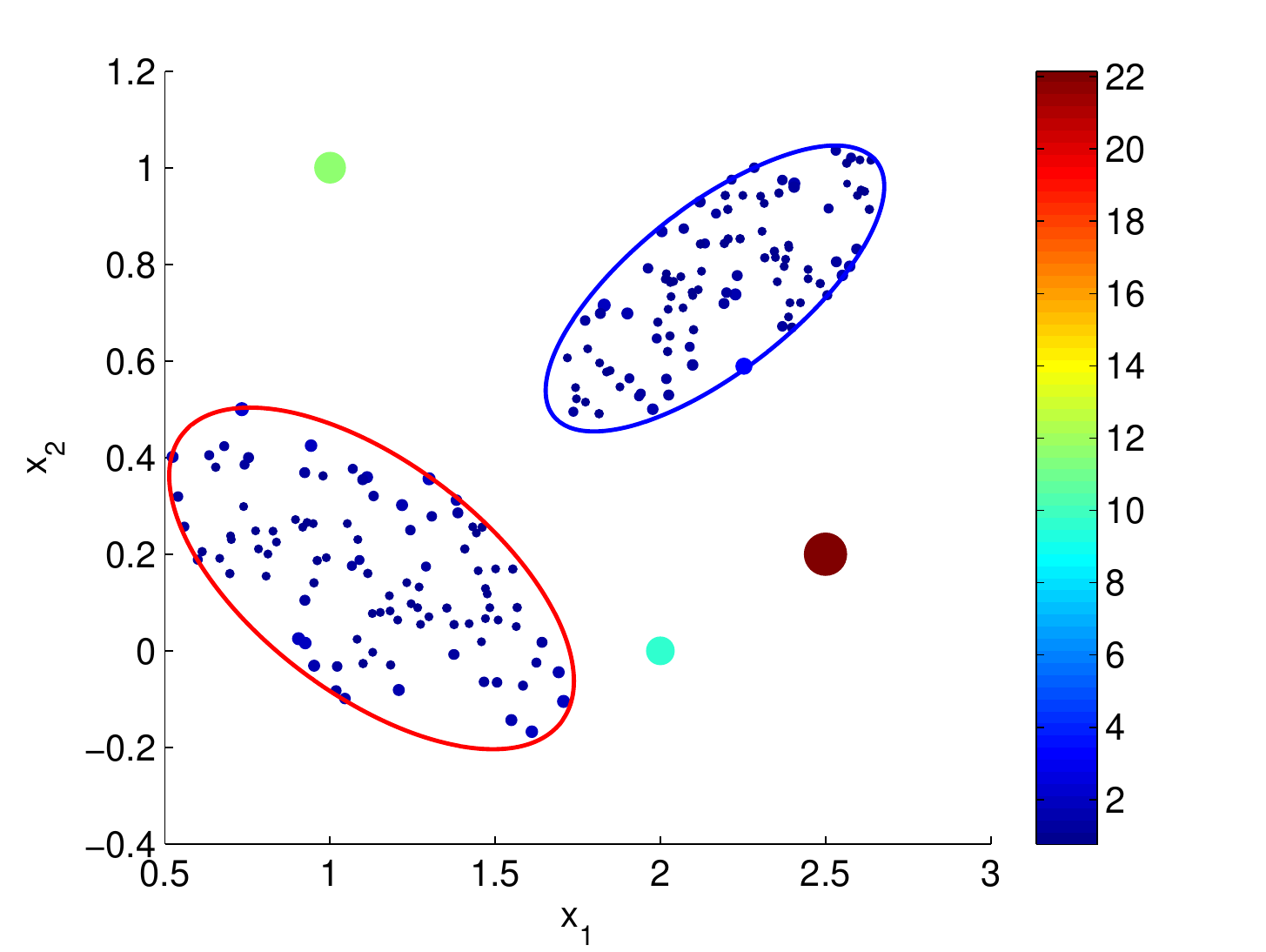}
\caption{The outlierness scores of all data samples, where the value of RDOS is illustrated by the color and the radius of the circle.}
\label{example1_res}
\end{figure}

The second synthetic data set used in our simulation consists of data samples uniformly distributed around a cosine curve, which can be written as
\begin{align}
x_2 = cos(x_1) + w
\end{align}
where $w \sim \mathcal{N}(0, \sigma^2)$. In our simulation, we use $\sigma^2 = 0.1$, and generate four outliers in this data set, as shown in Fig. \ref{example2}. The RDOS value of all data samples is shown in Fig. \ref{example2_res}, where both the color and the radius of circles indicate the value of RDOS. It is still shown that the RDOS-based method can effectively detect the outliers. 

\begin{figure}[!ht]
\centering
\includegraphics[width= 10 cm]{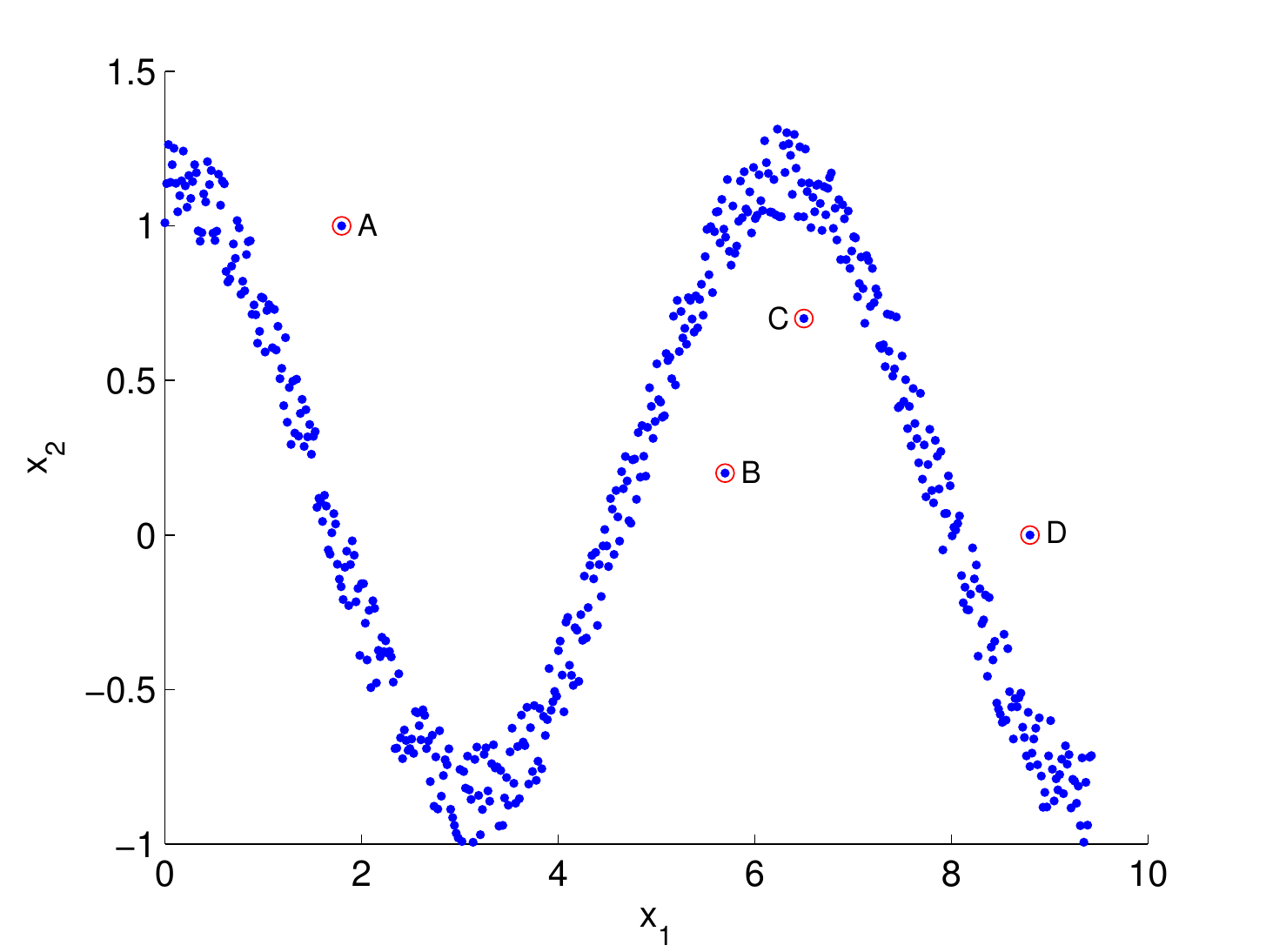}
\caption{The distribution of normal data and outliers, where $A$, $B$, $C$ and $D$ are considered as outliers. }
\label{example2}
\end{figure}

\begin{figure}[!ht]
\centering
\includegraphics[width= 10 cm]{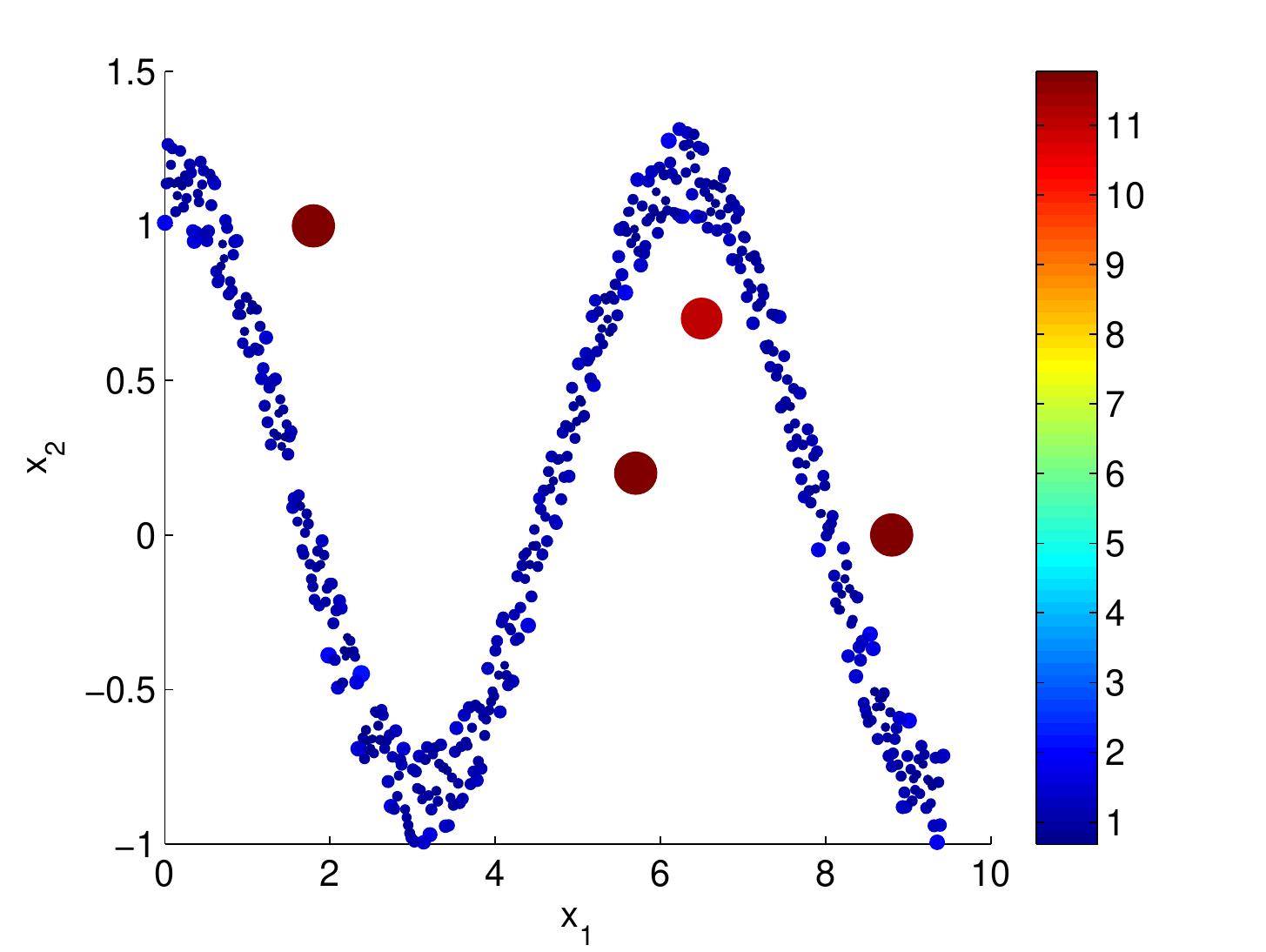}
\caption{The outlierness scores of all data samples, where the value of RDOS is illustrated by the color and the radius of the circle.}
\label{example2_res}
\end{figure}

\subsection{Real-Life Data Sets}
We also conduct outlier detection experiments on four real-life data sets to demonstrate the effectiveness of our proposed RDOS approach. All of these four data sets are originally from the UCI repository \cite{asuncion2007uci}, including \textsc{Breast Cancer}, \textsc{Pen-Local}, \textsc{Pen-Global}, and \textsc{Satellite}, but are modified for local and global outlier detection \cite{dataset}. We summarize the characteristics of these four data sets in Table \ref{tab_datasets}. Prior to calculating the RDOS, we first normalize the data ranging from 0 to 1. In Fig. \ref{data_examples}, we show the first two principle components of these four data sets, where the outliers are denoted by the solid circle. 

\begin{figure}[!ht]
\centering
\includegraphics[width= 10 cm]{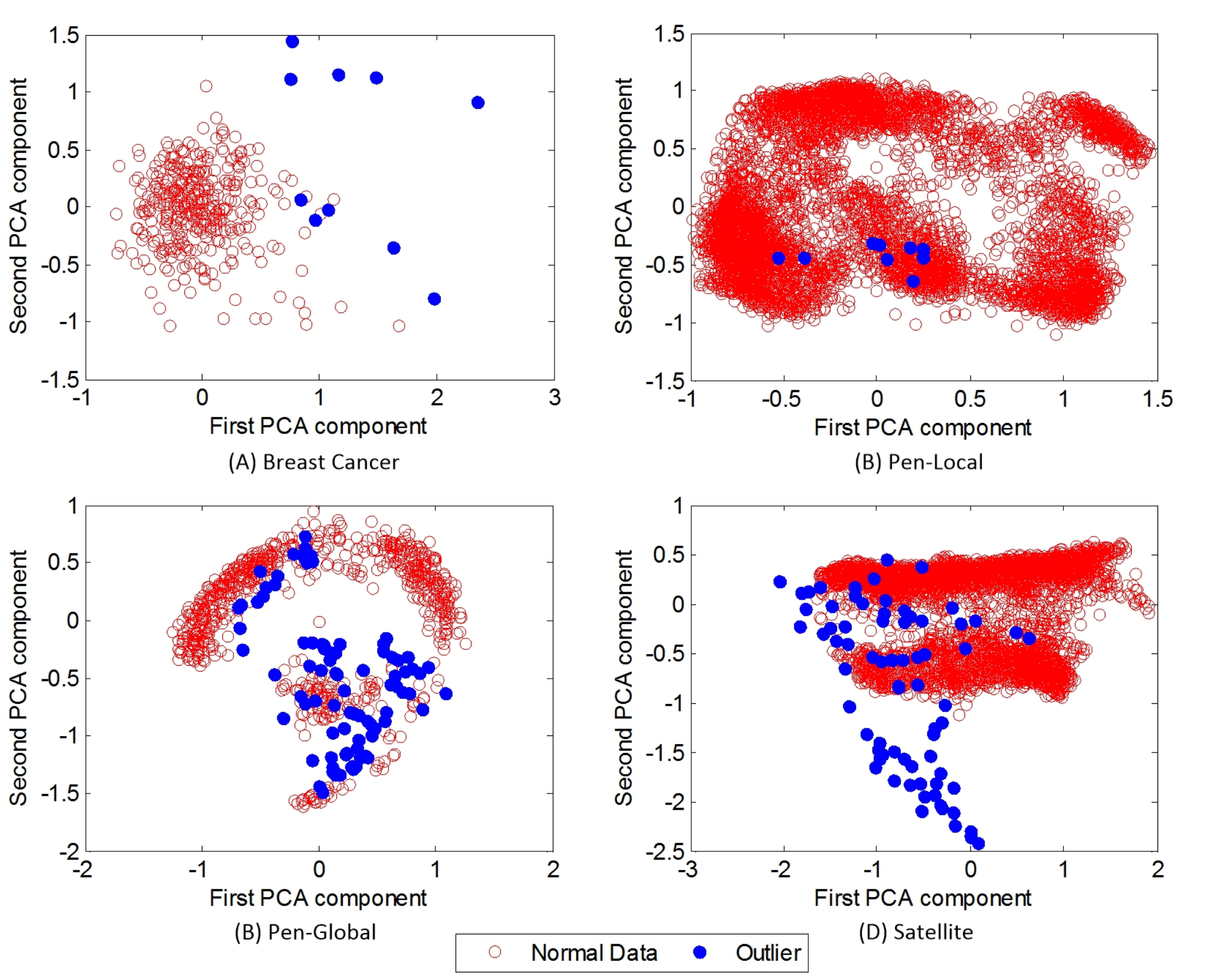}
\caption{The normal data and outliers in four real-life data sets: (A) \textsc{Breast Cancer}, (B) \textsc{Pen-Local}, (C) \textsc{Pen-Global}, and (D) \textsc{Satellite}. Only the first two principle components are shown. }
\label{data_examples}
\end{figure}

\begin{table}[!t]
\centering
\renewcommand {\arraystretch}{1.5}
\caption{The characteristics of four data sets}
\label{tab_datasets}
\centering
\begin{tabular}{l c c c c }
\Xhline{2\arrayrulewidth}
Dataset & $\#$ of features & $\#$ of outliers & $\#$ of data \\
\Xhline{2\arrayrulewidth}
\textsc{Breast Cancer} & $30$ & $10$ & $357$\\
%		\hline
\textsc{Pen-Local} & $16$ & $10$ & $6714$ \\
%		\hline
\textsc{Pen-Global}  & $16$ & $90$ & $719$ \\
%		\hline
\textsc{Satellite} & $36$ & $75$ & $5025$ \\
\Xhline{2\arrayrulewidth}
\end{tabular}
\end{table}

For each data sample, we calculate its RDOS and compare it with a threshold to determine whether it is an outlier. Since all these data sets are highly imbalanced, the use of overall accuracy is not appropriate. In our experiments, we use the metric of AUC (area under the ROC curve) for performance comparison. The ROC curve examines the performance of a binary classifier with different thresholds, leading to different pairs of false alarm rate and true positive rate. We compare our RDOS approach with another four widely used outlier detection approaches: Outlier Detection using Indegree Number (ODIN) \cite{Hautamaki:2004}, LOF \cite{breunig2000lof}, INFLO \cite{jin2006ranking}, and Mutual Nearest Neighbors (MNN) \cite{brito1997connectivity}. Since all of these examined methods are nearest neighbors-based methods, we evaluate the oultier detection performance with different $k$ values. Fig. \ref{breast_cancer_h1} shows the performance of AUC for the data set of \textsc{Breast Cancer}. It can be shown that our proposed RDOS approach, in general, performs better than other four approaches, and has a similar performance to the approaches of LOF and INFLO when $k$ is larger than 7. When $k=5$, the performance improvement of the proposed RDOS approach is largest, as illustrated in Fig. \ref{breast_cancer_roc}. 

%\begin{figure*}[hbp]
%\captionsetup[subfigure]{labelformat=empty}
%  \centering
%  \subcaptionbox{(AI). AUC in \textsc{Breast Cancer} \label{fm}}{\includegraphics[width=4.2cm]{breast_cancer_h1}}
%  \subcaptionbox{(iia).  \label{gm}}{\includegraphics[width=4.2cm]{pen_local_h1}}  
%  \subcaptionbox{(C1). ?. \label{gm}}{\includegraphics[width=4.2cm]{pen_global_h1}}  
%  \subcaptionbox{(D1). ?. \label{gm}}{\includegraphics[width=4.2cm]{satellite_h1}}
%
%\subcaptionbox{(AII). ROC in \textsc{Breast Cancer}. \label{fm}}{\includegraphics[width=4.2cm]{breast_cancer_roc}}
%  \subcaptionbox{(B2). ?. \label{gm}}{\includegraphics[width=4.2cm]{pen_local_roc}}  
%  \subcaptionbox{(C2). ?. \label{gm}}{\includegraphics[width=4.2cm]{pen_global_roc}}  
%  \subcaptionbox{(D2). ?. \label{gm}}{\includegraphics[width=4.2cm]{satellite_roc}}
%  
%  \caption{The testing F-Measure and G-Mean of our PPT-based classifier using class-specific features, compared with the conventional approaches, when IG and MD criteria are used for feature selection on the data sets of \textsc{20-Newsgroups} and \textsc{Reuters-10}.}
%  \label{fm_gm}
%\end{figure*}

\begin{figure}[!ht]
\centering
\includegraphics[width= 10 cm]{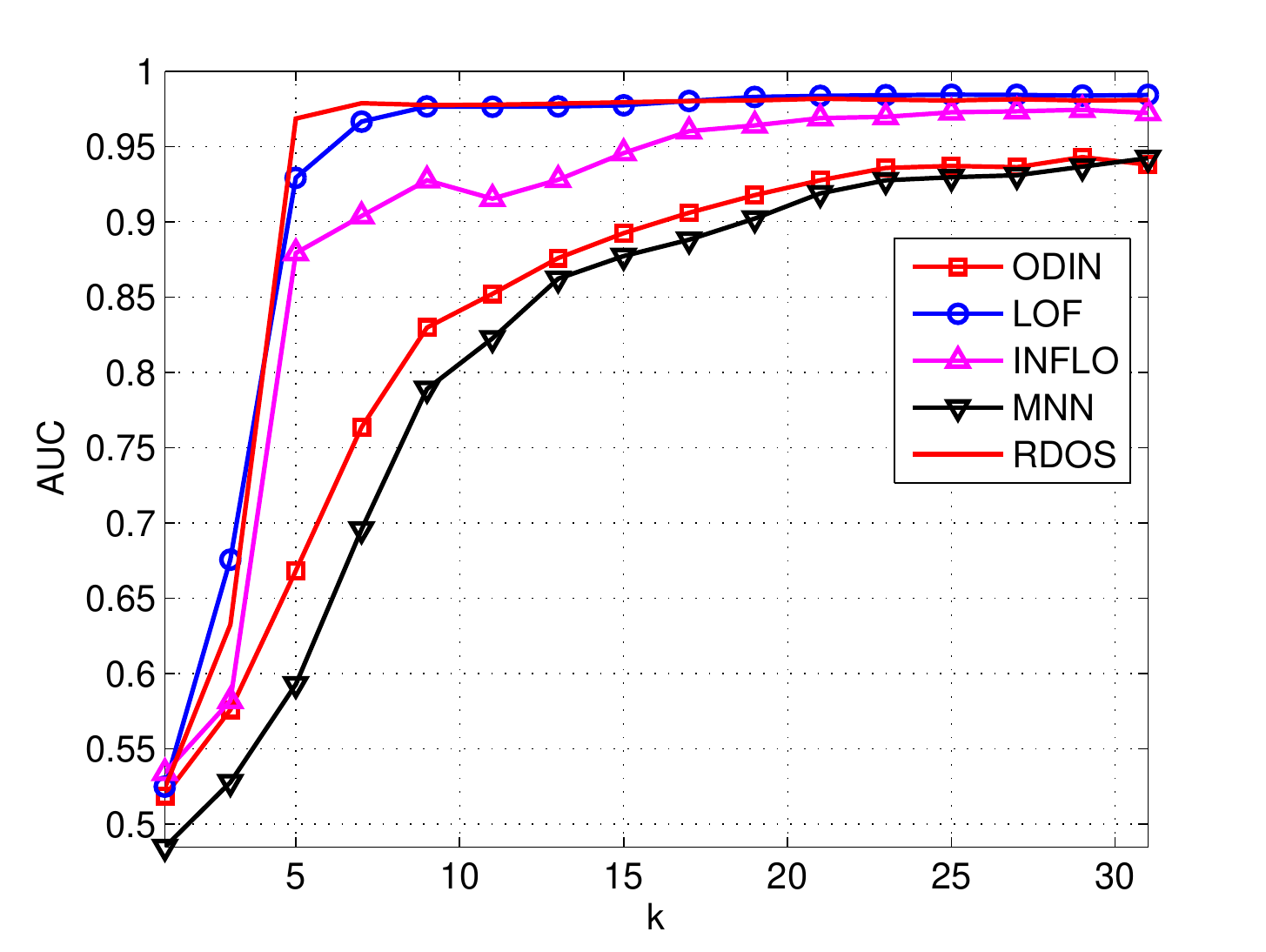}
\caption{The performance of AUC with different $k$ values for the data set of \textsc{Breast Cancer}}
\label{breast_cancer_h1}
\end{figure}

\begin{figure}[!ht]
\centering
\includegraphics[width= 10 cm]{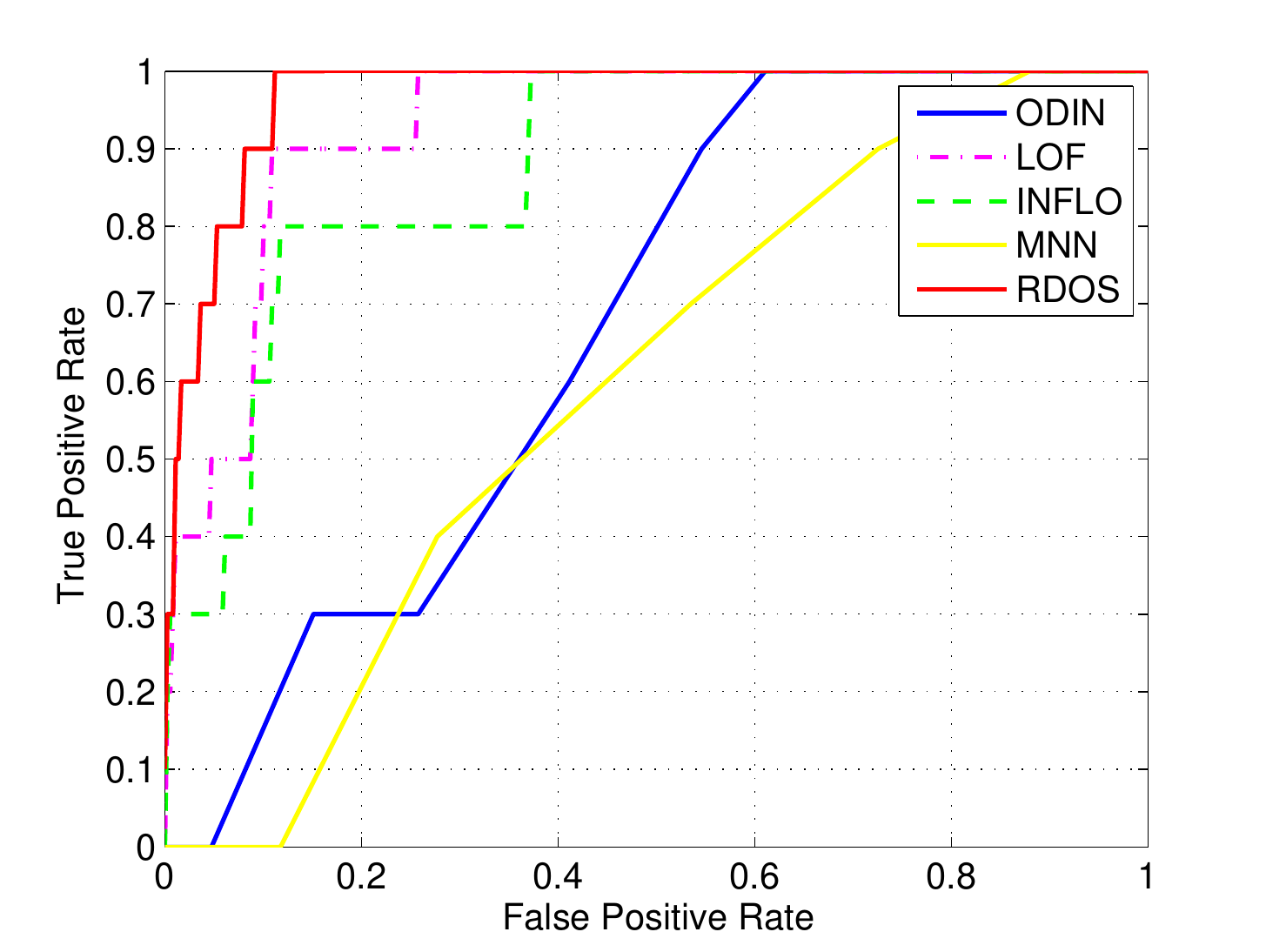}
\caption{The ROC for the data set of \textsc{Breast Cancer}, where $k = 5$}
\label{breast_cancer_roc}
\end{figure}

In Fig. \ref{pen_local_h1}, we show the performance of AUC for the data set of \textsc{Pen-Local}. It also shows that our RDOS approach generally outperforms other four approaches  when $k$ is less than 7. Specifically, in Fig. \ref{breast_cancer_roc}, we show the ROC curve for $k = 5$. Compared to the LOF and INFLO approaches, the performance difference is close to zero for a large $k$ value.  

\begin{figure}[!ht]
\centering
\includegraphics[width= 10 cm]{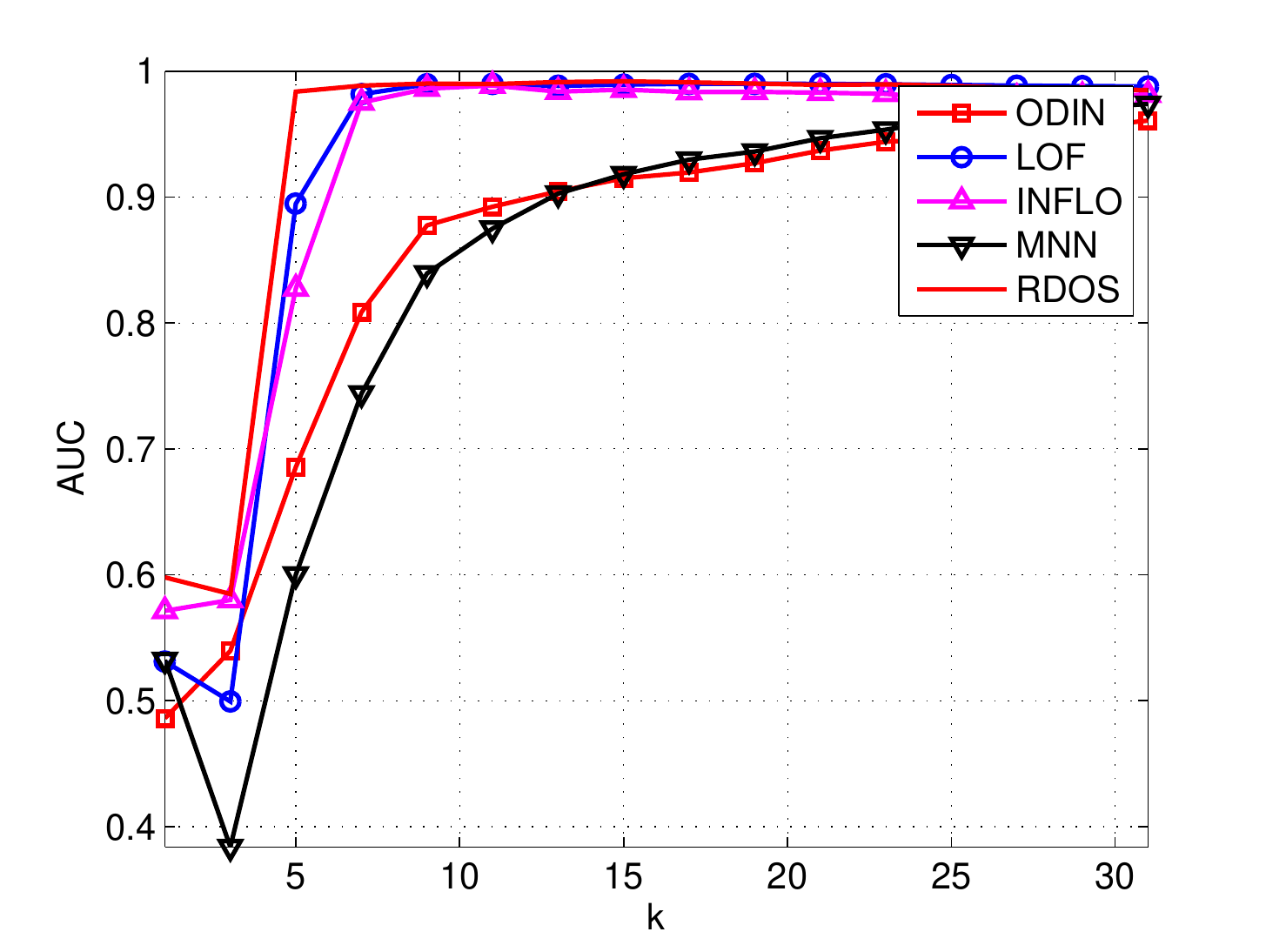}
\caption{The performance of AUC with different $k$ values for the data set of \textsc{Pen-Local}}
\label{pen_local_h1}
\end{figure}

\begin{figure}[!ht]
\centering
\includegraphics[width= 10 cm]{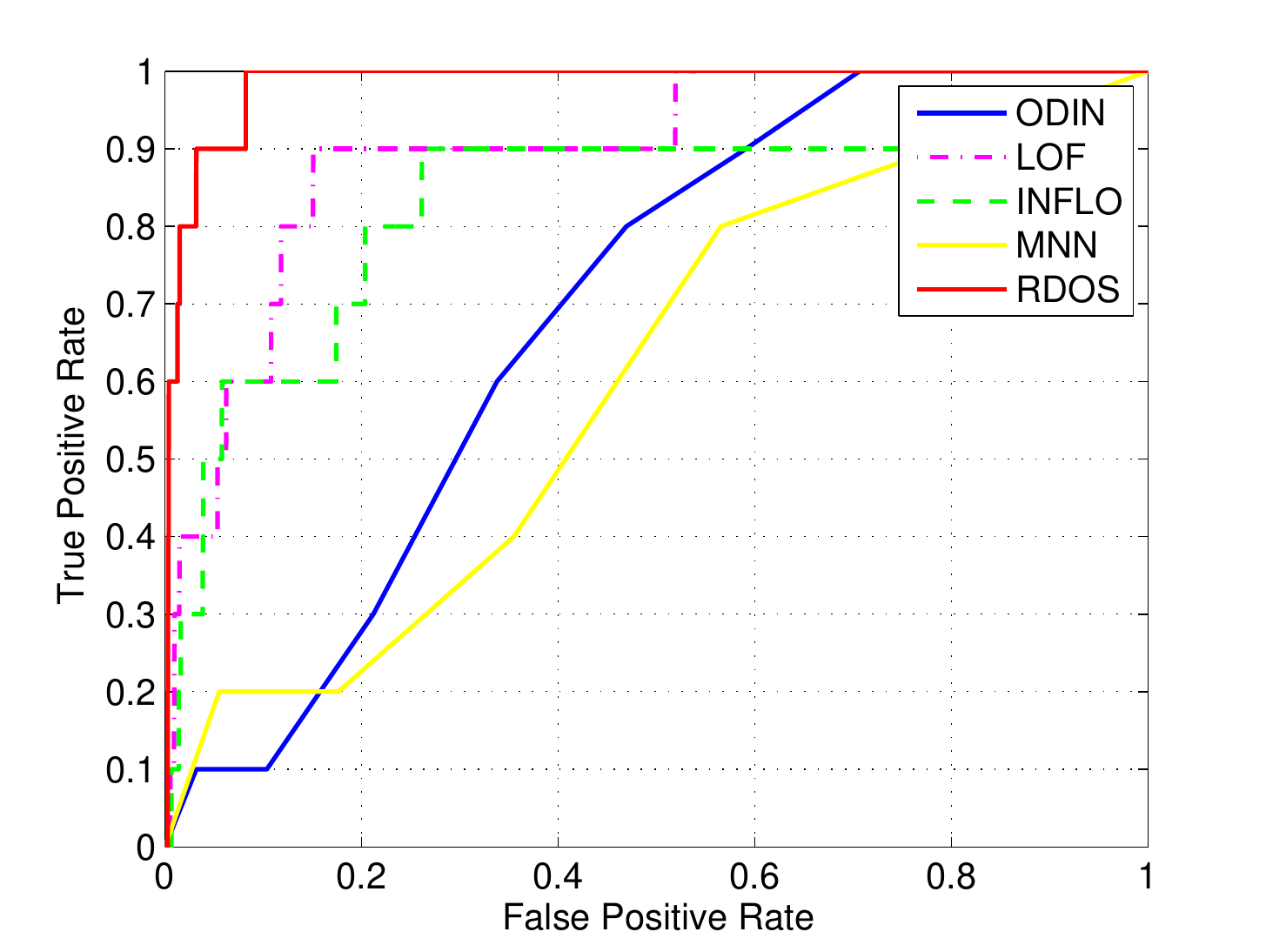}
\caption{The ROC for the data set of \textsc{Pen-Local}, where $k = 5$}
\label{pen_local_roc}
\end{figure}

In Fig. \ref{pen_global_h1}, we show the performance of AUC for the data set of \textsc{Pen-Global}. It shows a large performance improvement of our RDOS approach when the number of nearest neighbors $k$ increases. In Fig. \ref{pen_global_roc}, the ROC curves of all the five approaches are compared, when $k = 15$. From Fig. \ref{breast_cancer_h1}, \ref{pen_local_h1} and \ref{pen_global_h1}, it can be shown that RDOS $>$ LOF $>$ INFLO $>$ ODIN $>$ MNN, where the symbol ``$>$" means ``performs better than", for the data sets of \textsc{Breast Cancer}, \textsc{Pen-Local}, and \textsc{Pen-Global}. 

\begin{figure}[!ht]
\centering
\includegraphics[width= 10 cm]{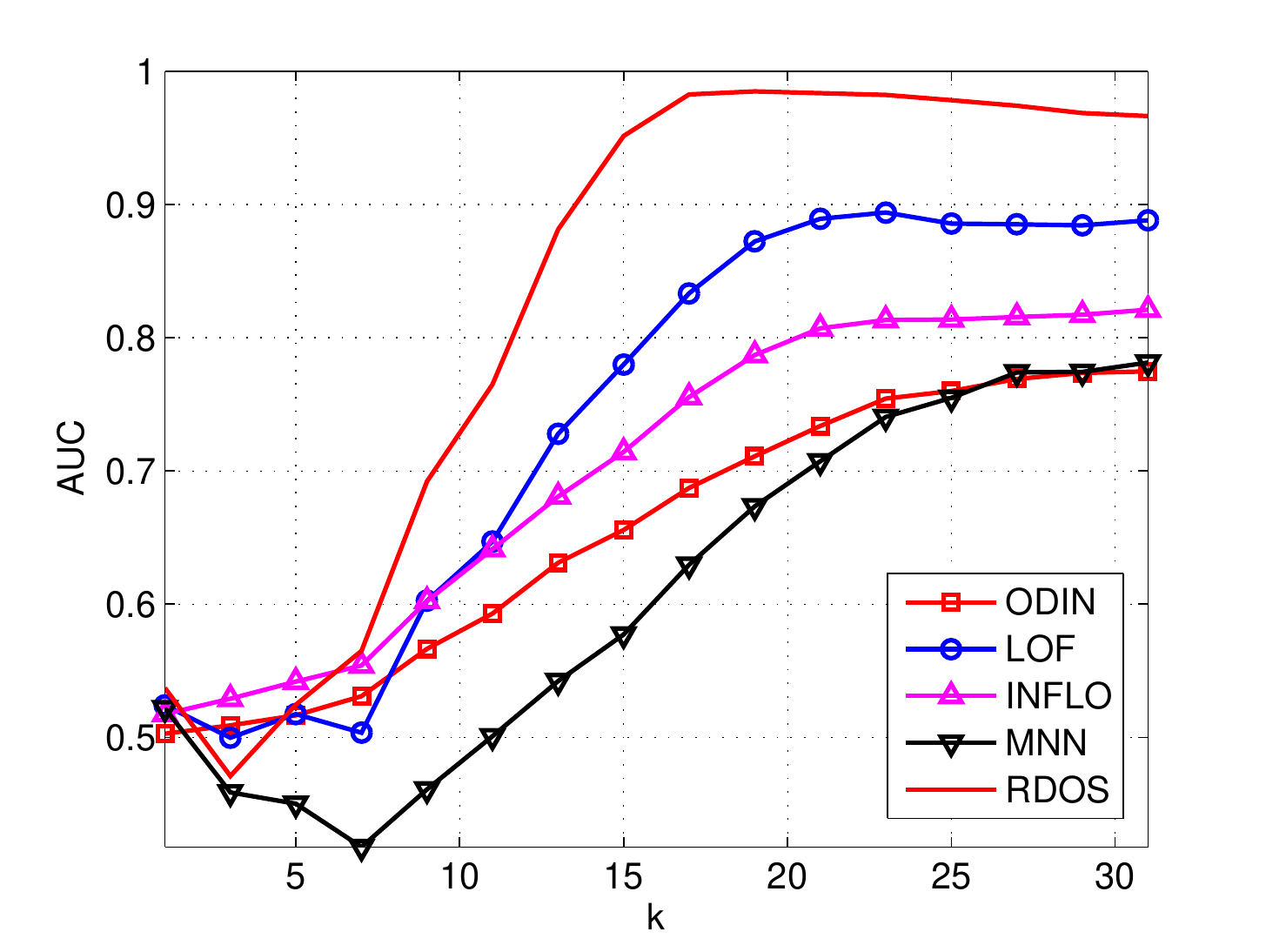}
\caption{The performance of AUC with different $k$ values for the data set of \textsc{Pen-Global}}
\label{pen_global_h1}
\end{figure}

\begin{figure}[!ht]
\centering
\includegraphics[width= 10 cm]{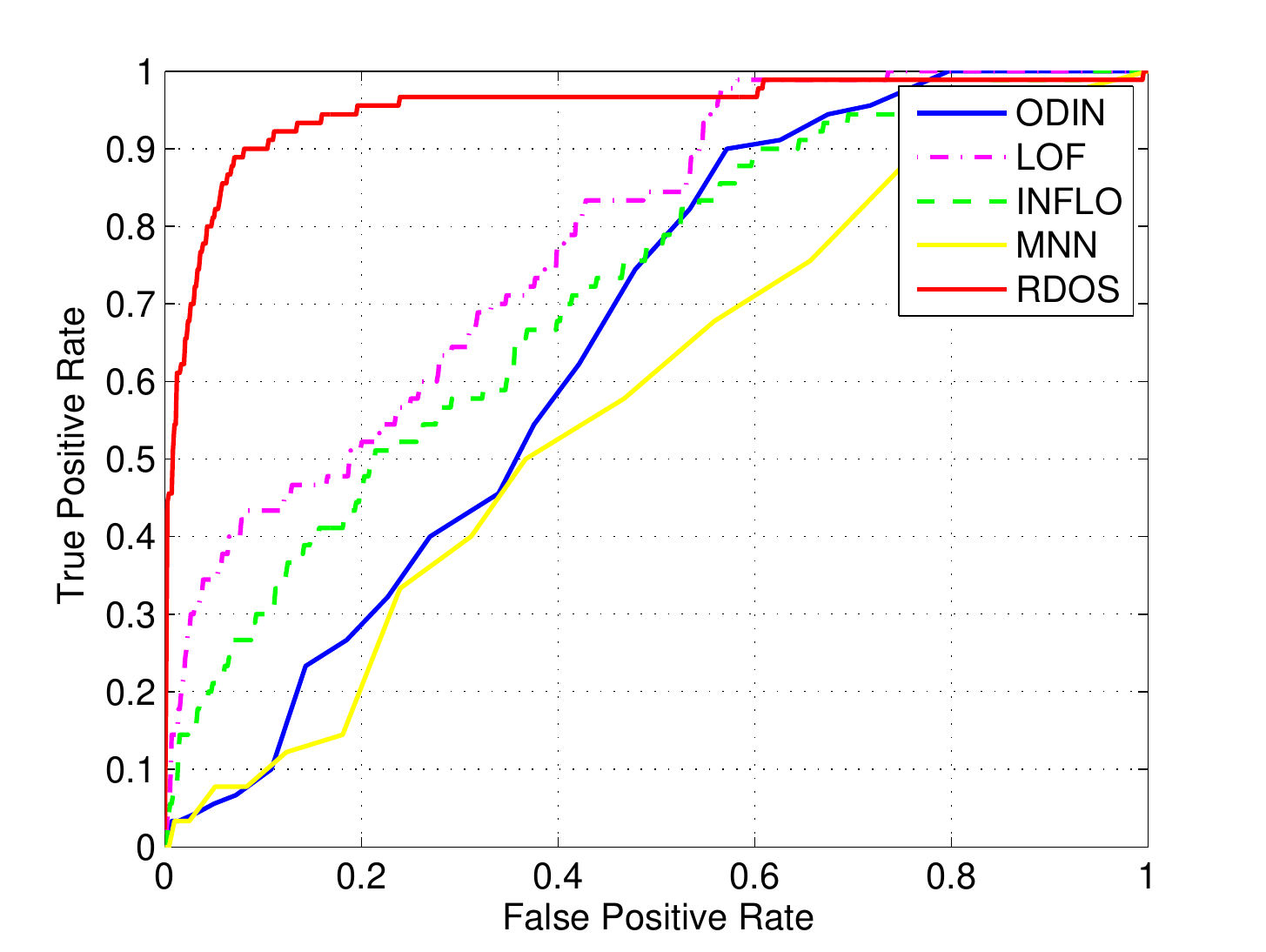}
\caption{The ROC for the data set of \textsc{Pen-Global}, where $k = 15$}
\label{pen_global_roc}
\end{figure}

Fig. \ref{satellite_h1} shows the performance of AUC for the data set of \textsc{Satellite}. When the number of nearest neighbors $k$ is less than 11, three approaches of RDOS, LOF and INFLO have a similar AUC. When the number of nearest neighbors $k$ is larger than 11 and less than 23, the approach of INFLO performs the best and our RDOS approach is the second. When the number of nearest neighbors $k$ is larger than 25, our RDOS approach has the best performance. Specifically, we show the ROC curve of all the five approaches in Fig. \ref{satellite_roc}. In general, we observe the following phenomena in our experiments: Firstly, the performance of all the five approaches is usually poor for a small $k$, and the improvement of our RDOS approach is not significant. When a small number of nearest neighbors are considered, the relative density in a neighborhood might not be well represented. Secondly, the proposed RDOS approach performs the best for specific $k$ values. Thirdly, we observe that the MNN approach has the worst performance, compared with other four approaches, for these four data sets. 

\begin{figure}[!ht]
\centering
\includegraphics[width= 10 cm]{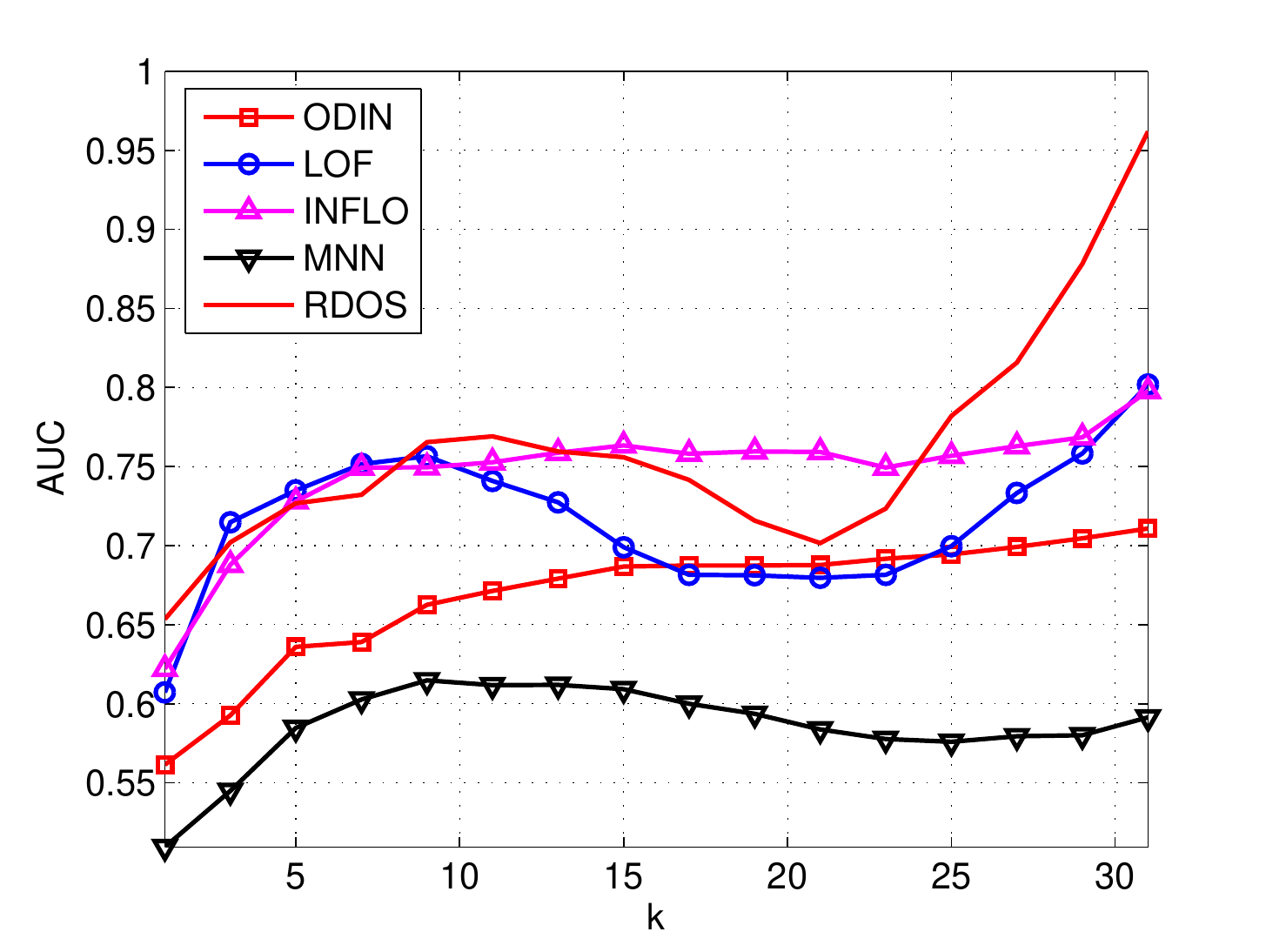}
\caption{The performance of AUC with different $k$ values for the data set of \textsc{Satellite}}
\label{satellite_h1}
\end{figure}

\begin{figure}[!ht]
\centering
\includegraphics[width= 10 cm]{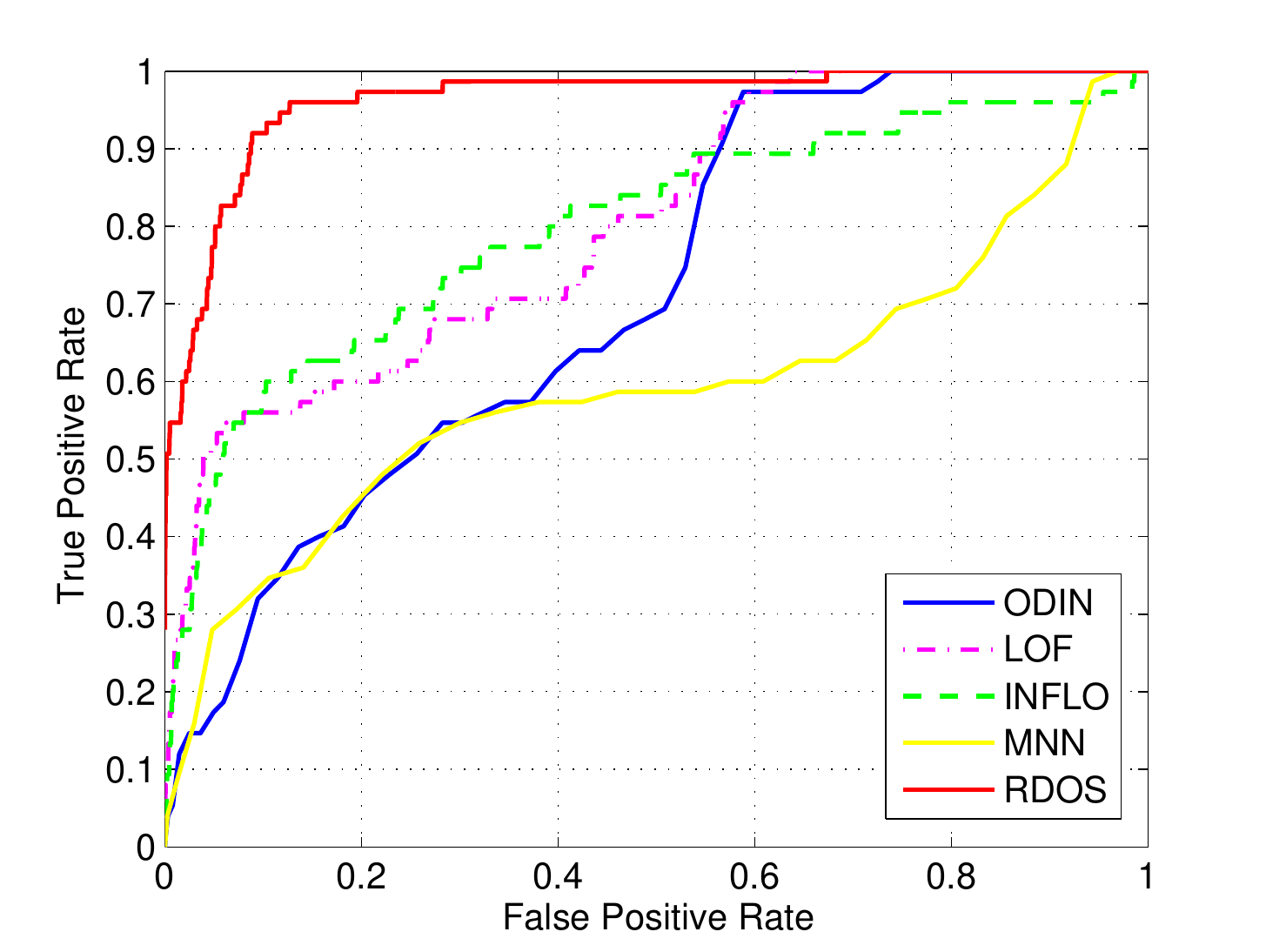}
\caption{The ROC for the data set of \textsc{Satellite}, where $k = 31$}
\label{satellite_roc}
\end{figure}

%\begin{figure*}[htp]
%\captionsetup[subfigure]{labelformat=empty}
%  \centering
%  \subcaptionbox{(A). \textsc{Breast Cancer} \label{breast_cancer_h1}}{\includegraphics[width=5cm]{breast_cancer_h1}}\hspace{0.1mm}%
%  \subcaptionbox{(B). \textsc{Pen-Local} \label{pen_local_h1}}{\includegraphics[width=5cm]{pen_local_h1}}\hspace{0.1em}%
%  \subcaptionbox{(C). \textsc{Pen-Global} \label{pen_global_h1}}{\includegraphics[width=5cm]{pen_global_h1}}\hspace{0.1em}%
%  \subcaptionbox{(D). \textsc{Satellite} \label{satellite_h1}}{\includegraphics[width=4cm]{satellite_h1}}\hspace{0.1em}%
%
%  \caption{The performance of AUC with different $k$ for the data set of (A) \textsc{Breast Cancer}, (B) \textsc{Pen-Local}, (C) \textsc{Pen-Global}, and (D) \textsc{Satellite}.}
%  \label{res_all}
%\end{figure*}

\section{Conclusions and Future Work}
This paper presented a novel local outlier detection method based on local kernel density estimation. Instead of only considering the $k$ nearest neighbors of a data sample, we considered three kinds of neighbors: $k$ nearest neighbors, reverse nearest neighbors, and shared nearest neighbors, for local kernel density estimation. A simple but efficient relative density calculation, termed Relative Density-based Outlier Score (RDOS), was introduced to measure the outlierness. We further derived theoretical properties of the proposed RDOS measure, including the expected value and the false alarm probability. The theoretical results suggest parameter settings for practical applications. Simulation results on both synthetic data sets and real-life data sets illustrate superior performance of our proposed method. One drawback of kernel-based density estimation is its kernel width selection. Along this research direction, new density estimation methods such as exponentially embedded families \cite{tangEEF,tangtoward,kayprob} and PDF projection theorem \cite{baggenstoss2003pdf,tangEEF2016} will be investigated in our future work. \\

\noindent \textbf{Reference:}
%\end{document}  % This is where a 'short' article might terminate

%ACKNOWLEDGEMENTS are optional

\bibliographystyle{elsarticle-num}
\bibliography{ref}  % sigproc.bib is the name of the Bibliography in this case
% ACM needs 'a single self-contained file'!
% SIGKDD: balancing columns messes up the footers: Sunita Sarawagi, Jan 2000.
% \balancecolumns
% \appendix
%Appendix A

% That's all folks!
\end{document}

%% file: Macros.tex
\newcounter{appcount}
\newcommand{\appendicesname}
            {Appendix\ \thechapter  \Alph{appcount}}
\newcommand{\bookappendicesname}
            {Appendix\ \Alph{appcount}}
\newcommand{\chapterappendix}[1]
          {\par\setcounter{section}{0}
           \setcounter{equation}{0}
           \setcounter{table}{0}
           \setcounter{figure}{0}
          \addtocounter{appcount}{1}   \renewcommand{\theequation}{\thechapter\Alph{appcount}.\arabic{equation}}
          \renewcommand{\thetable}{\thechapter\Alph{appcount}.\arabic{table}}
          \renewcommand{\thefigure}{\thechapter\Alph{appcount}.\arabic{figure}}
           \setcounter{section}{\arabic{chapter}\Alph{section}}
           \if@openright\cleardoublepage\else\clearpage\fi
           \chapter*{\huge{\appendicesname}\newline\newline \Huge{#1}}
           \addcontentsline{toc}{section}{\thechapter\Alph{appcount} #1}
           \markright{\MakeUppercase{\appendicesname.\ { #1}}}}
\newcommand{\bookappendix}[1]
          {\par\setcounter{section}{0}
           \setcounter{equation}{0}
           \setcounter{table}{0}
           \setcounter{figure}{0}
          \addtocounter{appcount}{1}   \renewcommand{\theequation}{\Alph{appcount}.\arabic{equation}}
           \renewcommand{\thetable}{\Alph{appcount}.\arabic{table}}
             \renewcommand{\thefigure}{\Alph{appcount}.\arabic{figure}}
           \setcounter{section}{\arabic{chapter}\Alph{section}}
           \if@openright\cleardoublepage\else\clearpage\fi
           \chapter*{\huge{\bookappendicesname}\newline\newline \Huge{#1}}
           \addcontentsline{toc}{chapter}{\bookappendicesname #1}
          \markright{\MakeUppercase{\bookappendicesname.\ { #1}}} }
\newcounter{example}
\newcounter{property}
\newcommand{\ben}{\begin{equation}}
\newcommand{\een}{\end{equation}}
\newcommand{\bea}{\begin{eqnarray*}}
\newcommand{\eea}{\end{eqnarray*}}
\newcommand{\bean}{\begin{eqnarray}}
\newcommand{\eean}{\end{eqnarray}}